\documentclass[]{article}

\usepackage{amssymb,amsmath}
\usepackage{hyperref}

\usepackage{graphicx,subfigure}
\usepackage{epsfig}
\usepackage{hyperref}
\usepackage{amsmath}
\usepackage{amssymb}
\usepackage{algorithm}
\usepackage{algorithmic}
\usepackage{url}
\usepackage{enumerate}
\usepackage{amsfonts}
\usepackage{boxedminipage}
\usepackage{tikz}
\usepackage{mathabx}
\usepackage{tabularx,ragged2e,booktabs,caption}

\usepackage{fullpage}

\newcommand{\ignore}[1]{}

\usepackage{amsthm}
\usepackage{amsfonts}

\newtheorem{lemma}{Lemma}
\newtheorem{theorem}{Theorem}
\newtheorem{corollary}{Corollary}

\author{Naman Agarwal\footnote{Google AI Princeton} \and Alon Gonen\footnote{Department of Computer Science, Princeton University} \and Elad Hazan\footnote{Google AI Princeton and Princeton University}}

\title{Learning in Non-convex Games with an Optimization Oracle}

\begin{document}
\maketitle

\begin{abstract}
We consider  online learning in an adversarial, non-convex setting under the assumption that the learner has an access to an offline optimization oracle.  In the general setting of prediction with expert advice, \cite{Hazan} established that in the optimization-oracle model, online learning requires exponentially more computation than statistical learning.  In this paper we show that by slightly strengthening the oracle model, the online and the statistical learning models become computationally equivalent. Our result holds for any Lipschitz and bounded (but not necessarily convex) function.  As an application we demonstrate how the offline oracle enables efficient computation of an equilibrium in non-convex games, that include GAN (generative adversarial networks) as a special case.
\end{abstract}

\section{Introduction}
The setting of \emph{online learning in games} is a fundamental paradigm which allows formulation of tasks such as spam detection, online routing, online recommendation systems,  and more \cite{Cesa-Bianchi2006, Hazan2016c, Shalev-Shwartz2011a}. A key feature of this model is the ability of the environments to evolve over time, possibly in an adversarial manner. Consequently, this framework can be used to produce more robust learners compared to the classic stationary and statistical learning framework. A fundamental question investigated in recent literature is whether this robustness comes with a computational price. While it is well-known that any efficient online learner can be transformed into an efficient \emph{statistical} (or {\it batch}) learner \cite{bianchi04}, it is important to understand to what extent is the online model  harder. 

To enable a  systematic comparison between the two models we must allow a reduction in the opposite direction. To this end we adopt the \emph{offline optimization oracle} model suggested in \cite{Hazan}, where the online learner submits a sequence of loss functions and the oracle returns any minimizer of the cumulative loss. For the well-established setting of learning with expert advice, \cite{Hazan} demonstrated an exponential gap between the oracle complexity in the online and the statistical settings. 

In this paper we study the same question in the more general non-convex setting.\footnote{We explain how to reduce the expert setting to the non-convex setting in Section \ref{sec:noContradiction}.} Deviating from \cite{Hazan}, we allow the learner to linearly perturb the objective submitted to the oracle. Arguably, adding a linear term to a non-convex function should not increase the overall complexity of the oracle. Perhaps surprisingly, we show that this moderate modification renders the online adversarial setting computationally equivalent to the statistical setting. We show this by extending the powerful Follow-the-Perturbed-Leader (FTPL) meta-algorithm to the non-convex setting and derive a polynomial bound on its oracle complexity.

\subsection{Setting and Main Result}
\subsubsection{Basic definitions and assumptions}
Let $\mathcal{W} \subseteq \mathbb{R}^d$ be the decision set (a.k.a. hypothesis space in the statistical setting) with $\ell_\infty$-diameter at most $D$, and let $\mathcal{L} \subseteq \mathbb{R}^\mathcal{W}$ be the set of all $G$-Lipschitz functions w.r.t. the $\ell_1$-norm. We assume that both $G$ and $D$ are polynomial in the ambient dimension $d$.\footnote{The choice of norm in our setting is inconsequential as norms are equivalent up to $poly(d)$.}

Consider the setting of online learning, where an online algorithm predicts a point $w_t \in \mathcal{W}$ in iterative fashion and receives a feedback according to an adversarially chosen loss function $\ell_t \in \mathcal{L}$. The goal of the learner is to minimize the \emph{average regret}, which is defined as the difference between the average loss of the learner and that of the best fixed point $w^* \in \mathcal{W}$ in hindsight. We define the \textit{sample complexity} as the number of rounds required for attaining expected average regret at most $\epsilon$.

The statistical setting differs from the online setting in two important aspects. a) We assume that the loss functions are drawn according to some unknown fixed distribution. b) The learner receives a sample of loss functions drawn according to the same distribution. Then it has to output a single predictor $\hat{w}$. The goal of the learner is minimize the expected \textit{excess risk}, which is defined as $\mathbb{E}_{\ell}[\ell(\hat{w})]-\inf_{w \in \mathcal{W}} \mathbb{E}_{\ell}[\ell(w)]$. The sample complexity in this model is the size of a sample (of loss functions) that is required for attaining expected excess risk at most $\epsilon$. 

\subsubsection{The offline oracle model}
In order to compare between the online and the statistical models, we assume  an access to two types of oracles:
\begin{enumerate}
\item \textbf{Value oracle } whose input is a pair $(w,\ell) \in \mathcal{W} \times \mathcal{L}$ and its output is $\ell(w)$.
\item \textbf{Offline optimization oracle} whose input consists of a sequence of loss functions $(\ell_1,\ldots,\ell_k) \in \mathcal{L}^k$ and a $d$-dimensional vector $\sigma$,  and its output is output has the form
$$
\hat{w} \in \arg\!\min \{\sum_{i=1}^k \ell_i(w) - \sigma^\top w:~w \in \mathcal{W}\}~.
$$
\end{enumerate}
We define the oracle complexity as 
$$
\textrm{sample complexity} + \textrm{\# of calls to value oracle} + \textrm{\# of calls to offline oracle}
$$
\subsubsection{Main result}
Our online Algorithm  \ref{alg:main} applies the offline oracle with a random linear perturbation $\sigma$ whose coordinates are  i.i.d. exponential random variables with parameter $\eta$. Our main result can be stated as follows.
\begin{theorem} \label{thm:main} The oracle complexity of  Algorithm \ref{alg:main} is $\text{poly}(d,1/\epsilon)$.
\end{theorem}
Notably, both the loss functions and the domain $\mathcal{W}$ are not assumed to be convex. The oracle complexity in the statistical setting (under the same assumptions) is also $\mathrm{poly}(d,1/\epsilon)$.\footnote{This follows from standard covering argument} We thus conclude that both statistical and the online oracle complexities for non-convex learning setting are polynomially equivalent. 
We deduce the following game theoretic result:
\begin{corollary}(informal)  
Convergence to equilibrium in two player zero-sum non-convex games is as
hard as the corresponding offline best-response optimization problem.
\end{corollary}
We elaborate on this implication and specify it to GANs in Section \ref{sec:implications}.
\subsection{Related Work} \label{sec:related}

\paragraph{Follow-the-perturbed-leader. }
The ubiquitous Follow-the-Perturbed-Leader (FTPL) algorithm  \cite{hannan1957approximation,Kalai2004} is the canonical example of using an optimization oracle: the algorithm returns the result of a single optimization oracle call per iteration. 
Since its introduction, an extensive study of FTPL has yielded new insights and efficient variants in various different settings (e.g. \cite{hazan2012online,devroye2013prediction,van2014follow,cohen2015following}).

\paragraph{Online Convex Optimization. } 

If the problem admits a convex structure, then the oracle complexity is polynomial in the dimension via bandit convex optimization \cite{Cesa-Bianchi2006,Hazan2016c,bubeck2012regret}. If one considers the number of oracle calls to the optimization oracle only, and does not have access to a value oracle, then it is still possible to obtain a polynomial bound on the oracle complexity. This is due to the fact that online convex optimization reduces to online linear optimization \cite{zinkevich2003online}, and this enables extension of FTPL to the convex case.  However, this extension requires access to the gradient, which does not fall into our oracle model. We are not aware of any analysis of direct application of FTPL to a convex loss (i.e., without access to the gradients). In a sense, our treatment of the non-convex case gives the first direct analysis for FTPL to the convex case.

\paragraph{The experts setting: Overcoming the lower bound} \label{sec:noContradiction}
It is instructive to revisit the experts setting and understand why our result does not contradict the exponential lower bound of \cite{Hazan}. After all, one can easily embed the general experts problem in the $d$-dimensional hypercube for $d = \lceil \log \, N \rceil$ using the following standard technique:
\begin{enumerate}

\item Associate each vertex $z \in \{0,1\}^d$ with some expert $i(z)$.
\item Associate each $x \in [0,1]^d$ with a random expert according to $p(z) = \prod _{i=1}^d (z_i x_i + (1-z_i)(1-x_i))$.
\item Perform optimization over $[0,1]^d$, where the loss of each $x \in [0,1]^d$ is $\sum_{z \in \{0,1\}^d} p(z) \ell(i(z))$.
\end{enumerate}
It can be verified that the parameters $G$ and $D$ are polynomial in $d$, as required. Consequently, our main result applies to this setting as well.

Crucially, unlike our oracle model, \cite{Hazan} does not allow a linear perturbation of the cumulative loss in this low-dimensional presentation. As it seems, this arguably moderate modification of the model rendered the offline-to-online reduction tractable.

\paragraph{Experts with low-dimensional structure. } 
In the context of contextual bandits,  \cite{Dudik2017} formulate abstract conditions
under which the randomness can be shared between the experts, and allow efficient regret minimization in the oracle complexity model. 

\cite{gonen2017fast} study stability in  non-convex settings, and bound  the stability rate of ERM for strict saddle problems. In this paper we derive stable algorithms under much more moderate assumptions.

\paragraph{Generative adversarial networks.}
Several works have studied GANs in the regret minimization framework (e.g. \cite{schuurmans2016deep, Kodali, hazan2017efficient}). We provide the first evidence that achieving equilibrium in GANs can be reduced to the offline problems associated with the players.

\subsection{Overview and Techniques}
\subsubsection{Why standard approaches do not work?}
A common approach which works well in the convex setting is to apply the Follow-the-Regularized-Leader (FTRL) with $\ell_2$-regularization:
$$
w_t \in \arg\!\min \left\{ \sum_{i<t} \ell_i(w) + \eta \|w\|^2\right \},~~~\eta \approx T^{\alpha}, ~\alpha \in (0,1)~.
$$
In the convex case $\ell_2$-regularization stabilizes the solution by pushing it towards zero. However, we argue that in the non-convex setting, this approach does not help. To demonstrate this claim, consider a $1$-dimensional setting, where the loss functions have the form $w
\mapsto (\sigma(w x)-y)^2$, where $\sigma(x) =
\max\{x,0\}$ is the ReLU function and $x \in [-1,1], y \in [0,1]$. Due to the ReLU term, the magnitude of the loss incurred by classifying $x$ negatively is not important (i.e., there is no difference between $wx = -10^{-6}$ and $wx = -1$).  Informally, if all $x$'s are bounded away from zero, we mostly care for the ratio between positive and negative examples. Therefore, adding $\ell_2$-regularization does not make solutions near zero more appealing. It is not hard to formalize this argument and show that FTRL with $\ell_2$ (or $\ell_1$) regularization can not yield sublinear regret. 

\subsubsection{Extending FTPL to the non-convex case}
Our result is proved by extending the Follow-The-Perturbed-Leader algorithm to the non-convex setting. As we detail in the preliminaries section, online learnability requires algorithmic stability between consecutive rounds. For linear loss functions,  \cite{Kalai2004} proved that linear perturbation of the loss stabilized the loss function itself, and consequently the minimizer is stable as well. The proof relies heavily on the fact that the perturbation and the loss function are of the same type.

In the non-convex case, we can not hope to stabilize the loss itself using a linear perturbation. Nevertheless, our main contribution is to establish that the randomness injected by FTPL does stabilize the predictions of the learner. We prove this result by investigating how the outputs of FTPL change as we vary the the noise vector $\sigma \in \mathbb{R}^d_{\ge 0}$. In the $1$-dimensional case, this investigation yields a useful monotonicity property which helps us bounding the expected distance between consecutive minimizers. While the general $d$-dimensional introduces some challenges, we are able to effectively reduce the analysis to the $1$-dimensional setting by varying each coordinate of the noise separately.

\section{Preliminaries}

\subsection{Online to batch conversion}
The following well-known result due to \cite{bianchi04} tells us
that the online sample complexity dominates the batch sample complexity.
The intuition that online learning is at least as hard as batch learning
is formalized by the following online-to-batch theorem.

\begin{theorem} \cite{bianchi04} Suppose that $\mathcal{A}$ is an
online learner with
$\frac{\mathbb{E}[\mathrm{Regret}_T ]}{T} \le \epsilon(T)$ for any
$T$. Consider the following algorithm for the batch setting: given a
sample $(\ell_1,\ldots,\ell_n) \sim \mathcal{P}^n$, the algorithm
applies $\mathcal{A}$ to the sample in an online manner. Thereafter,
it draws a random round $j \in [n]$ uniformly at random and returns
$\hat{w} = w_j$. Then the expected excess risk of the algorithm,
$\mathbb{E}[L(\hat{w})] -L(w^\star)$, is at most
$\epsilon(T)$.\end{theorem}

\subsection{Online learning via stability}
The main challenge in online learning stems from the fact that the
learner has to make a decision before observing the adversarial action.
Intuitively, we expect that the performance after shifting the actions
of the learner by one step (i.e.~considering the loss
$\ell_t(w_{t+1})$ rather than $\ell_t(w_t)$) to be optimal. This
view suggests that online learning is all about balancing between
optimal performance w.r.t. previous rounds and ensuring stability
between consecutive rounds. Similarly to the statistical setting, the most common algorithmic tool for achieving stability is regularization. In particular, the well-established Follow-the-Regularized-Leader is a meta-algorithm whose instances are determined by choosing a concrete regularization function. Precisely, given a regularizer $R:\mathbb{R}^d 
\rightarrow \mathbb{R}$, the $t$-iterate of the algorithm is 
\[
w_t = \arg\!\min \left\{ \sum_{i<t} \ell_i(w) +R(w) \right\}~.
\]
The next well-known lemma provides a systematic approach for analyzing \emph{Follow-the-Regularized-Leader}-type algorithms.
\begin{lemma} \label{lem:FTL-BTL} (FTL-BTL \cite{Kalai2004}) The regret of Follow-the-Regularized-Leader is at most
\[
\mathbb{E} [\mathrm{Regret}_T] \le \mathbb{E} [R(w^\star)-R(w_1)] + \sum_{i=1}^T \mathbb{E}[\ell_t(w_t) - \ell_t(w_{t+1})] ~,
\]
where $w^\star = \arg\!\min \{\sum_{t=1}^T \ell_i(w):~w \in \mathcal{W} \}$.
\end{lemma}

\subsection{The exponential distribution}
We use the following properties of the exponential distribution.
\begin{lemma} Let $X$ be an exponential random variable with parameter
$\eta$.\footnote{That is, $X$ has density $p(x) = \eta \exp(-\eta x)$.} The following properties hold: a) for
any $s \in \mathbb{R}$, $P(X \ge s) = \exp(-\eta s)$. b)
Memorylessness: for any $s,q \in \mathbb{R}$,
$P(X \ge q+s|X \ge q) = P(X \ge s)$. c) if $X_1,\ldots,X_d$ are
i.i.d. with $X_i \sim Exp(\eta)$, then
$\mathbb{E}[\|(X_1,\ldots,X_d)\|_\infty] \le \eta^{-1}(\log(d)+1)$.\end{lemma}

\section{Non-convex FTPL}
In this section we present and analyze the non-convex FTPL method presented in  Algorithm \ref{alg:main}. Our analysis completes the proof of our main theorem (Theorem \ref{thm:main}). Along the proof we distinguish between the one-dimensional and the general $d$-dimensional case. For the former case we obtain better regret bound in terms of the dependence on the horizon parameter $T$. Omitted proofs are provided in the Appendix.

\begin{algorithm}
\caption{Non-convex FTPL}
\label{alg:main}
\begin{algorithmic}
\STATE Parameter: $\eta >0$
\FOR {$t=1$ \TO $T$}
\STATE Draw i.i.d. random vector $\sigma_t \sim (Exp(\eta))^d$
\STATE Prediction at time $t$:
\begin{equation}
\label{eqn:ftplstep}
\begin{aligned}
w_t \in \arg\!\min  \left \{ \sum_{i<t} \ell_i(w) - \sigma_t^\top w:\,w \in \mathcal{W} \right\}~,
\end{aligned}
\end{equation}
\ENDFOR
\end{algorithmic}
\end{algorithm}

\subsection{Reduction to oblivious setting}
To simplify the presentation we make the following standard modification:  \begin{enumerate}
\item
The adversary is oblivious in the sense that the sequence $(\ell_t)_{t=1}^T$ is chosen in advance. 
\item
This allows us to analyze a slightly different algorithm which draws only a single noise vector $\sigma \sim Exp(\eta)^d$ rather than drawing a fresh noise vector on every round.
\end{enumerate}
It follows from \cite{Cesa-Bianchi2006}[Lemma 4.1] that proving regret bounds for this variant translates into asymptotically equivalent (expected) regret bounds for non-oblivious adversaries using Algorithm \ref{alg:main}.
\subsection{Main Lemma}
Throughout this section we use the notation $w_t(\sigma)$ to emphasize that $w_t$ as defined in \eqref{eqn:ftplstep}, is determined by the noise vector $\sigma$. Following Lemma \ref{lem:FTL-BTL} we would like to to establish a bound on the expected instability at time $t$, i.e. $\mathbb{E}[ \ell_t(w_t(\sigma)) - \ell_t(w_{t+1}(\sigma)) ]$. This is bounded above by  $G \cdot \mathbb{E}\|w_t(\sigma) - w_{t+1}(\sigma)\|_1$. Note that the distance between $w_t$ and $w_{t+1}$ is ill-defined since both $w_t$ and $w_{t+1}$ are not unique. However, as we show below, we will be able to derive a uniform bound on the distance between any consecutive minimizers for \textbf{every} choice of minimizers. Note that this not really needed. As we are primarily interested in stability with respect to the function value, we can make any assumptions on the tie-breaking mechanism. However, we found it both interesting and surprisingly easier to prove the stronger result.

\begin{lemma} \label{lem:key}
Fix an iteration $t$ and let $\delta>0$ be a \emph{margin} parameter. There is a tie-breaking rule for choosing minimizers such that $\mathbb{E} [\|w_t (\sigma)- w_{t+1}(\sigma)\|_1]= O\left(\frac{\mathrm{poly}(d) \eta}{\delta} + d \delta \right)$. In the one-dimensional case we obtain the improved bound $\mathbb{E}[|w_t - w_{t+1}| ] = O(\eta)$.
\end{lemma}
\begin{proof} \textbf{(of Theorem \ref{thm:main})}
We start with the multidimensional case. Applying the FTL-BTL lemma (Lemma \ref{lem:FTL-BTL}) with the regularizer $R(w)=-\sigma^\top w$ and using H\"{o}lder inequality, we obtain
\begin{align*}
\mathbb{E} [\mathrm{Regret}_T ] &\le \mathbb{E} \left [\|\sigma\|_\infty \cdot \|w^\star-w_1\|_1 \right] + G\sum_{t=1}^T \mathbb{E}[\|w_t(\sigma) - w_{t+1}(\sigma)\|_1] \\
& \le \mathbb{E} \left[\|\sigma\|_\infty \right] D + G\sum_{t=1}^T \mathbb{E}[\|w_t(\sigma) - w_{t+1}(\sigma)\|_1]~.
\end{align*}
For the multidimensional case, we use that $\mathbb{E} [\|\sigma\|_\infty] \le \eta^{-1}(\log d+1)$, $D,G \in \mathrm{poly}(d)$, and apply Lemma \ref{lem:key} to obtain
$$
\begin{aligned}
\mathbb{E}[\mathrm{Regret}_T] \le \mathrm{poly}(d) \left((\eta^{-1}(\log d + 1)  + T (\eta\delta^{-1}+\delta))\right)~.
\end{aligned}
$$
By setting $\eta = T^{-2/3}$ and $\delta = T^{-1/3}$, we obtain the regret bound $\mathbb{E}[\mathrm{Regret}_T] \le O(T^{2/3}\mathrm{poly}(d))$. Online-to-batch conversion yields a sample complexity bound of $O \left(\frac{\mathrm{poly}(d)}{\epsilon^3} \right)$. 

In the $1$-dimensional case we simply set $\eta = T^{-1/2}$ to obtain $\mathbb{E}[\mathrm{Regret}_T] = O(T^{1/2})$. This translates into a sample complexity bound of $O\left(\frac{\mathrm{poly}(d)}{\epsilon^2} \right)$.
\end{proof}

\subsection{Proof of Lemma \ref{lem:key}}
We begin with the following lemma which provides a bound on the gap between minimizers with respect to the change in noise parameter $\sigma$. 

\begin{lemma}
\label{lemma:optimalitylemma}
For any two functions $f_1,f_2:\mathcal{W} \rightarrow \mathbb{R}$ and vectors $\sigma_1, \sigma_2 \in \mathbb{R}^d$, let \[
w_i(\sigma_i) \in \arg\!\min \left\{f_i(w) - \sigma_i^{\top}w \right\},\qquad i=1,2.
\]
Letting $f=f_1-f_2$ and $\sigma=\sigma_1-\sigma_2$, we have that
\begin{equation}\label{eqn:1}
f(w_1(\sigma_1)) - f(w_2(\sigma_2)) \leq \sigma^{\top}(w_1(\sigma_1) - w_2(\sigma_2))
\end{equation}
\end{lemma}

\begin{proof}
Using optimality conditions for $w_i(\sigma_i)$, we have that 
\[f_1(w_1(\sigma_1)) - \sigma_1^T w_1(\sigma_1) \leq f_1(w_2(\sigma_2)) - \sigma_1^T w_2(\sigma_2) \]
\[f_2(w_2(\sigma_2)) - \sigma_2^T w_2(\sigma_2) \leq f_2(w_1(\sigma_1)) - \sigma_2^T w_1(\sigma_1) \]
Adding the above two inequalities and rearranging finishes the proof.
\end{proof}
We now provide the proof of Lemma \ref{lem:key} in the two considered cases. 
\begin{proof} \textbf{(of Lemma \ref{lem:key}: one-dimensional case)}
We wish to use Lemma \ref{lemma:optimalitylemma}. For any time $t$, consider substituting $f_1(w) = \sum_{i < t} l_i(w)$, $f_2(w) = \sum_{i < t+1} l_i(w)$, $\sigma_2 = \sigma$ and $\sigma_1 = \sigma' \triangleq \sigma + 2G$. We immediately get that 
\[ l_{t}(w_t(\sigma')) - l_{t}(w_{t+1}(\sigma)) \leq 2G(w_t(\sigma') - w_{t+1}(\sigma))\]
Using the fact that $l_{t}$ is $G$-lipschitz, we get that 
\[-G |w_t(\sigma') - w_{t+1}(\sigma)| \leq l_{t}(w_t(\sigma')) - l_{t}(w_{t+1}(\sigma)) \leq 2G(w_t(\sigma') - w_{t+1}(\sigma))\]
which immediately implies that $w_t(\sigma') \geq w_{t+1}(\sigma)$. Similar calculations show that $w_{t+1}(\sigma') \geq w_t(\sigma)$ and $w_{t}(\sigma') \ge w_t(\sigma)$.

For the rest of the proof we will omit the dependence on $t$ as it will be clear from context. We denote by $w_{\min}(\sigma) = \min \{w_t(\sigma), w_{t+1}(\sigma\} \}$, $w_{\max}(\sigma) = \max \{w_t(\sigma), w_{t+1}(\sigma)\}$. First we observe that
$$
\mathbb{E} [|w_t(\sigma) - w_{t+1}(\sigma)|] = \mathbb{E} [w_{\max}(\sigma)] -  \mathbb{E} [w_{\min}(\sigma)] ~.
$$
Secondly the computation above implies that
\begin{equation}
\label{eqn:minmaxbound_1dim}
w_{\min}(\sigma') \geq w_{\max}(\sigma)~.
\end{equation}
This powerful monotonicity property (see Figure \ref{fig:monotone}) is now used to lower bound $\mathbb{E}[w_{\min}(\sigma)]$ in terms $\mathbb{E}[w_{\max}(\sigma)]$. Letting $\sigma'=\sigma + 2G$, we have
$$
\begin{aligned}
&\mathbb{E} [w_{\min}(\sigma)]  = \int_{\sigma=0}^{2G} \eta \exp(-\eta \sigma) w_{\min}(\sigma) \,d \sigma  + \int_{\sigma>2G} \eta \exp(-\eta \sigma) w_{\min}(\sigma) \,d \sigma\\
& \ge (1-\exp(-2\eta G)) (\mathbb{E}[w_{\max}(\sigma)]-D) + \int_{\sigma>0} \eta \exp(-\eta (\sigma')) w_{\min}(\sigma')\,d\sigma\\
&  \ge (1-\exp(-2\eta G)) (\mathbb{E}[w_{\max}(\sigma)]-D) +  \int_{\sigma>0} \eta \exp(-\eta (\sigma')) w_{\max}(\sigma) \,d\sigma\\
& =  (1-\exp(-2\eta G)) (\mathbb{E}[w_{max}(\sigma)]-D) + \exp(-2\eta G) \mathbb{E} [w_{\max}(\sigma)] \\
& = \mathbb{E}[w_{\max}(\sigma)] - D(1-\exp(-2\eta G)) \ge \mathbb{E}[w_{\max}(\sigma)] - 2\eta D G~,
\end{aligned}
$$
where the second inequality uses Equation \ref{eqn:minmaxbound_1dim} and the last inequality uses the inequality $\exp(x) \ge 1+x$. 
\end{proof}
\begin{figure}
\centering
\includegraphics[width=1\textwidth]{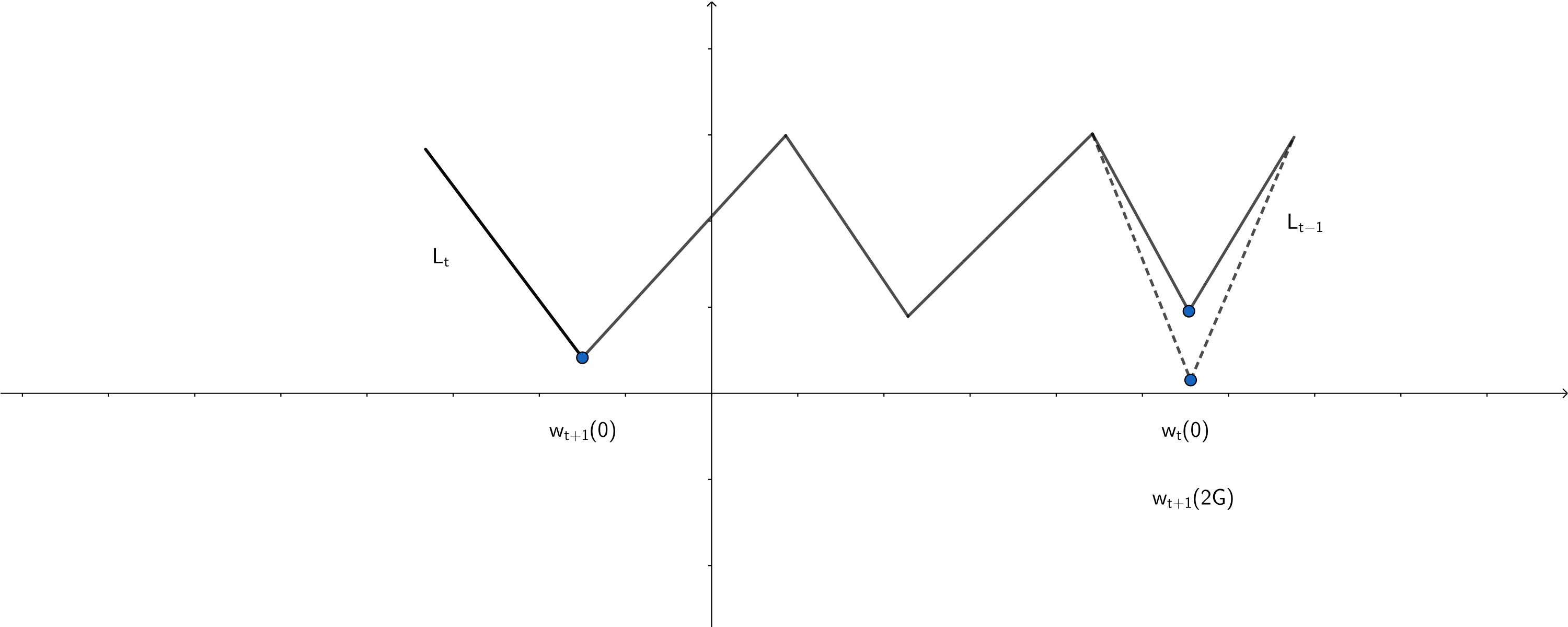}
\caption{Illustration of the monotonicity property used in the Proof of Lemma \ref{lem:key}: The unperturbed minimizer of $L_t$ (solid line), denoted $w_{t+1}(0)$, can be significantly smaller than the unperturbed minimizer of $L_{t-1}$ (dashed line), $w_t(0)$. This can be balanced by increasing the noise parameter corresponding to $w_{t+1}$.}
\label{fig:monotone}
\end{figure}

\begin{proof} \textbf{(of Lemma \ref{lem:key}:multiple dimensions)}
Once again we wish to use Lemma \ref{lemma:optimalitylemma}. For any time $t$ and any coordinate $k$, consider substituting $f_1(w) = \sum_{i < t} l_i(w)$, $f_2(w) = \sum_{i < t+1} l_i(w)$, $\sigma_2 = \sigma$ and $\sigma_1 = \sigma' \triangleq \sigma + 3B\delta^{-1} \cdot e_k$ (where $e_k$ is the $k^{th}$ vector in the canonical basis). We immediately get that 
\[ 
l_{t}(w_t(\sigma')) - l_{t}(w_{t+1}(\sigma)) \leq 3B\delta^{-1}(w_{t,k}(\sigma') - w_{t+1,k}(\sigma))~,
\]
where $w_{t,k}$ is the $k$-th coordinate of $w_t$. Using the fact that the range of $l_{t}$ is $[-B,B]$, we get that 
\[-2B \leq l_{t}(w_t(\sigma')) - l_{t}(w_{t+1}(\sigma)) \leq 3B\delta^{-1}(w_{t,k}(\sigma') - w_{t+1,k}(\sigma))\]
which immediately implies that $w_{t,k}(\sigma') \geq w_{t+1,k}(\sigma) - \delta$. A similar calculation also derives that $w_{t+1,k}(\sigma') \geq w_{t,k}(\sigma) - \delta$.

Now for any $k \in [d]$, let $w_{k,\min}(\sigma) = \min \{w_{t,k}(\sigma), w_{t+1,k}(\sigma) \}$, $w_{\max}(\sigma) = \max \{w_{t,k}(\sigma), w_{t+1,k}(\sigma)\}$. First we observe that
$$
\begin{aligned}
\mathbb{E} [\|w_t(\sigma) - w_{t+1}(\sigma)\|_1] = \sum_{k=1}^d( \mathbb{E} [w_{k,\max}(\sigma)] -  \mathbb{E} [w_{k,\min}(\sigma)]
\end{aligned}
$$
Secondly the calculation above implies that for all $k$
\begin{equation}
\label{eqn:minmaxbound_kdim}
  w_{k,\min}(\sigma + 3B\delta^{-1}e_k) \geq w_{k,\max}(\sigma ) - \delta  
\end{equation}

Now fix a coordinate $k \in [d]$ along with all noise coordinates $\sigma_j$ for $j \neq k$. Denote by $\mathbb{E}_{-k}$ the corresponding conditional expectation. Up to the additional margin term $\delta$, lower bounding $\mathbb{E}_{-k}[w_{k,\min}]$ in terms of $\mathbb{E}_{-k}[w_{k,\max}]$  reduces to the one-dimensional case; letting $q = 3B\delta^{-1}$ and $\mu(x) = \eta \exp(-\eta x)$, we have
$$
\begin{aligned}
&\mathbb{E}_{-k} [w_{k,\min}(\sigma_k)]  = \int_{\sigma_k=0}^{q} \mu(\sigma_k) w_{k,\min}(\sigma_k) \,d \sigma_k
+ \int_{\sigma_k>q} \mu(\sigma_k)w_{k,\min}(\sigma_k) \,d \sigma_k  \\
&\ge  (1-\exp(-q\eta)) (\mathbb{E}_{-k}[w_{k,\max}(\sigma_k)]-D)
+ \int_{\sigma_k>0} \mu (\sigma_k+q)w_{k,\min}(\sigma_k+q)\, d \sigma_k\\
&  \ge (1-\exp(-q\eta)) (\mathbb{E}_{-k}[w_{k,\max}(\sigma_k)]-D) + \int_{\sigma_k>0}  \mu (\sigma_k+q) (w_{k,\max}(\sigma_k)-\delta) \,d\sigma_k\\
& =  (1-\exp(-q\eta)) (\mathbb{E}_{-k}[w_{k,\max}(\sigma_k)]-D) + \exp(-q\eta) (\mathbb{E}_{-k} [w_{k,\max}(\sigma_k)] - \delta) \\
& \ge \mathbb{E}_{-k}[w_{k,\max}(\sigma_k)] - D(1-\exp(-q\eta))  - \delta \ge \mathbb{E}_{-k}[w_{k,\max}(\sigma_k)] - 3B\eta\delta^{-1}D -\delta~.
\end{aligned}
$$
The second inequality uses Equation \ref{eqn:minmaxbound_kdim} and the last inequality follows by substituting $q=3B\delta^{-1}$ and using the inequality $\exp(x) \ge 1+x$. Since the above holds for any fixed $\sigma_{-k} = (\sigma_j)_{j \neq k}$ , the unconditioned expectations also satisfy
$$
\mathbb{E}[w_{k,\min}(\sigma)] \ge \mathbb{E}[w_{k,\max}(\sigma)] - \frac{\mathrm{poly}(d) \eta}{\delta} - \delta~.
$$Summing over all coordinates we conclude the bound.
\end{proof}

\section{Implications to Non-convex Games}\label{sec:implications}
Consider the following formulation of a non-convex zero-sum game.  Let $F:\mathcal{X} \times \mathcal{Y} \rightarrow \mathbb{R}$, where $\mathcal{X},\mathcal{Y} \subseteq \mathbb{R}^d$ are compact with diameter at most $D$. The $x$-th player wishes to minimize $F$ and whereas the $y$-th player wishes to maximize $F$. We assume that for all $x\in \mathcal{X}$ and $y \in \mathcal{Y}$, both $F(\cdot,y)$ and $-F(x,\cdot)$ are $G$-Lipschitz and $B$-bounded. A known approach for achieving equilibrium is to apply (for each of the players) an online method with vanishing average regret. Precisely, on each round $t$ both players choose a pair $(x_t,y_t)$ which induces the losses $F(x_t,y_t)$ and $-F(x_t,y_t)$, respectively. Finally, we draw a random index $[j]\in [T]$ and output the pair $(\hat{x},\hat{y}) \triangleq (x_j,y_j)$. By endowing the players with
access to an offline oracle and playing according to non-convex FTPL we can reach approximate equilibrium.
\begin{theorem} \label{thm:onlineToNZ} Suppose that both the $x$-player and the $y$-player
have an access to an offline oracle and play according to non-convex
FTPL (Algorithm \ref{alg:main}). Given $\epsilon > 0$, let $T\in \mathrm{poly}(d)/\epsilon^3$ such that the expected average regret of non-convex FTPL is at most $\epsilon$. Then, $(\hat{x},\hat{y})$ forms an $\epsilon$-approximated equilibrium, i.e., for any $x \in \mathcal{X}$ and $y \in \mathcal{Y}$,
\[
\mathbb{E}[F(\hat{x},\hat{y})] \le \mathbb{E}[F(x,\hat{y})] +\epsilon,~~\mathbb{E}[F(\hat{x},\hat{y})] \ge \mathbb{E}[F(\hat{x},y)] -\epsilon~.
\]
\end{theorem}
Note that the players can use their offline oracle to amplify their confidence and achieve an equilibrium with high probability. The proof is provided in the appendix.

\subsection{Implication to GANs}
In particular, we consider the case where the $x$-th player is a
\emph{generator}, who produces synthetic samples (e.g.~images), whereas the $y$-th player acts as a \emph{discriminator} by assigning scores to
samples reflecting the probability of being generated from the true
distribution. Formally, by choosing a parameter $x \in \mathcal{X}$
and drawing a random noise $z$, the $x$-th player produces a sample
denote $G_x(z)$. Conversely, the $y$-th player chooses a parameter
$y \in \mathcal{Y}$ and assign the score $D_y(G_x(z)) \in [0,1]$ to
the sample $G_x(z)$. The function $F$ usually corresponds to the
log-likelihood of mistakenly assigning an high score to a synthetic
example and vice versa. It is reasonable to assume that $F$ is
Lipschitz and bounded w.r.t. the network parameters. As a result,
efficient convergence to GANs is established by assuming an access to an
offline oracle.

\section{Discussion}
Our work establishes a computational equivalence between online and
statistical learning in the non-convex setting. We shed
light on the hardness result of \cite{Hazan} by demonstrating that
online learning is significantly more difficult than statistical
learning only when no structure is assumed.

One interesting direction for further investigation is to refine the comparison model and study the polynomial dependencies more carefully. One obvious question is to understand the gap in terms of the horizon parameter $T$ between the regret bounds for the one-dimensional and the multidimensional settings.

\section*{Acknowledgements}
We thank Karan Singh for recognizing a bug in our original proof and several discussions. We also thank Alon Cohen and Roi Livni for fruitful discussions.
Elad Hazan acknowledges funding from NSF award Number 1704860.

\bibliographystyle{plain}
\bibliography{library.bib}

\begin{thebibliography}{10}

\bibitem{bubeck2012regret}
S{\'e}bastien Bubeck, Nicolo Cesa-Bianchi, et~al.
\newblock Regret analysis of stochastic and nonstochastic multi-armed bandit
  problems.
\newblock {\em Foundations and Trends{\textregistered} in Machine Learning},
  5(1):1--122, 2012.

\bibitem{bianchi04}
Nicol{\`{o}} Cesa-Bianchi, Alex Conconi, and Claudio Gentile.
\newblock {On the Generalization Ability of Online Learning Algorithms for
  Pairwise Loss Functions}.
\newblock {\em IEEE Transactions on Information Theory}, 50:2050----2057, 2004.

\bibitem{Cesa-Bianchi2006}
Nicolo Cesa-Bianchi and Gabor Lugosi.
\newblock {\em {Prediction, Learning, and Games}}.
\newblock Cambridge university press, 2006.

\bibitem{cohen2015following}
Alon Cohen and Tamir Hazan.
\newblock Following the perturbed leader for online structured learning.
\newblock In {\em International Conference on Machine Learning}, pages
  1034--1042, 2015.

\bibitem{devroye2013prediction}
Luc Devroye, G{\'a}bor Lugosi, and Gergely Neu.
\newblock Prediction by random-walk perturbation.
\newblock In {\em Conference on Learning Theory}, pages 460--473, 2013.

\bibitem{Dudik2017}
Miroslav Dudik, Nika Haghtalab, Haipeng Luo, Robert~E. Schapire, Vasilis
  Syrgkanis, and Jennifer~Wortman Vaughan.
\newblock {Oracle-efficient online learning and auction design}.
\newblock In {\em Annual Symposium on Foundations of Computer Science -
  Proceedings}, 2017.

\bibitem{gonen2017fast}
Alon Gonen and Shai Shalev-Shwartz.
\newblock {Fast Rates for Empirical Risk Minimization of Strict Saddle
  Problems}.
\newblock In Ohad Shamir and Satyen Kale, editors, {\em Proceedings of the 2017
  Conference on Learning Theory}, pages 1043----1063. PMLR, 2017.

\bibitem{hannan1957approximation}
James Hannan.
\newblock Approximation to bayes risk in repeated play.
\newblock {\em Contributions to the Theory of Games}, 3:97--139, 1957.

\bibitem{Hazan2016c}
Elad Hazan.
\newblock {Introduction to Online Convex Optimization}.
\newblock {\em Foundations and Trends{\textregistered} in Optimization},
  2(3-4):157--325, 2016.

\bibitem{hazan2012online}
Elad Hazan and Satyen Kale.
\newblock Online submodular minimization.
\newblock {\em Journal of Machine Learning Research}, 13(Oct):2903--2922, 2012.

\bibitem{Hazan}
Elad Hazan and Tomer Koren.
\newblock {The Computational Power of Optimization in Online Learning}.
\newblock In {\em Proceedings of the forty-eighth annual ACM symposium on
  Theory of Computing}, pages 128--141. ACM, 2016.

\bibitem{hazan2017efficient}
Elad Hazan, Karan Singh, and Cyril Zhang.
\newblock Efficient regret minimization in non-convex games.
\newblock In {\em International Conference on Machine Learning}, pages
  1433--1441, 2017.

\bibitem{Kalai2004}
Adam Kalai and Santosh Vempala.
\newblock {Efficient Algorithms for Online Decision Problems}.
\newblock {\em Journal of Computer and System Sciences}, 2004.

\bibitem{Kodali}
Naveen Kodali, Jacob Abernethy, James Hays, and Zsolt Kira.
\newblock {On Convergence and Stability of GANs}.
\newblock {\em arXiv preprint arXiv:1705.07215}, 2017.

\bibitem{schuurmans2016deep}
Dale Schuurmans and Martin~A Zinkevich.
\newblock Deep learning games.
\newblock In {\em Advances in Neural Information Processing Systems}, pages
  1678--1686, 2016.

\bibitem{Shalev-Shwartz2011a}
Shai Shalev-Shwartz.
\newblock {Online Learning and Online Convex Optimization}.
\newblock {\em Foundations and Trends{\textregistered} in Machine Learning},
  4(2):107--194, 2011.

\bibitem{van2014follow}
Tim Van~Erven, Wojciech Kot{\l}owski, and Manfred~K Warmuth.
\newblock Follow the leader with dropout perturbations.
\newblock In {\em Conference on Learning Theory}, pages 949--974, 2014.

\bibitem{zinkevich2003online}
Martin Zinkevich.
\newblock Online convex programming and generalized infinitesimal gradient
  ascent.
\newblock In {\em Proceedings of the 20th International Conference on Machine
  Learning (ICML-03)}, pages 928--936, 2003.

\end{thebibliography}

\newpage

\appendix

\section{Omitted Proofs}
\begin{proof} \textbf{(of Theorem \ref{thm:onlineToNZ})}
For all $y \in \mathcal{Y}$,
\[
\begin{aligned}
& \mathbb{E}[F(\hat{x},y)] = \mathbb{E}  \left [\frac{1}{T} \sum F(x_t,y) \right] \le \mathbb{E}  \left [\frac{1}{T} \sum F(x_t,y_t) \right]+\epsilon~\\
& \Rightarrow \mathbb{E}[F(\hat{x},\hat{y})] \le  \left [\frac{1}{T} \sum F(x_t,y_t) \right]+\epsilon
\end{aligned}
\]
Similarly, for all $x \in \mathcal{X}$,
\[
\begin{aligned}
&\mathbb{E}[F(x,\hat{y})] = \mathbb{E}  \left [\frac{1}{T} \sum F(x,y_t) \right] \ge \mathbb{E}  \left [\frac{1}{T} \sum F(x_t,y_t) \right]-\epsilon~\\
&\Rightarrow   \mathbb{E}[F(\hat{x},\hat{y})] \ge  \left [\frac{1}{T} \sum F(x_t,y_t) \right]-\epsilon~.
\end{aligned}
\]
\end{proof}

\end{document}